\documentclass{article}

\usepackage[preprint]{neurips_2024}

\usepackage[utf8]{inputenc} 
\usepackage[T1]{fontenc}    
\usepackage{hyperref}       
\usepackage{url}            
\usepackage{booktabs}       
\usepackage{amsfonts}       
\usepackage{nicefrac}       
\usepackage{microtype}      
\usepackage{xcolor}         
\usepackage{amsmath}
\usepackage{amssymb}
\usepackage{amsthm}
\usepackage{mathtools}
\usepackage{bm}
\usepackage{algorithm}
\usepackage{algpseudocode}
\usepackage{multicol}
\usepackage{graphicx}
\usepackage{subcaption}
\usepackage{multicol}
\usepackage{enumitem}
\newtheorem{theorem}{Theorem}
\newtheorem{lemma}{Lemma}
\newtheorem{definition}{Definition}
\newtheorem{corollary}{Corollary}
\newtheorem{assumptions}{Assumptions}
\usepackage{float}
\newcommand*\samethanks[1][\value{footnote}]{\footnotemark[#1]}

\title{GeoAdaLer: Geometric Insights into Adaptive Stochastic Gradient Descent Algorithms}

\author{
  Chinedu ~Eleh\thanks{Authors conrtibuted equally.} \\
  Department of Mathematics and Statistics\\
  Auburn University\\
  Auburn, AL 36849 \\
  \texttt{cae0027@auburn.edu} \\
  \And
  Masuzyo ~Mwanza\samethanks \\
  Department of Mathematics and Statistics\\
  Auburn University\\
  Auburn, AL 36849 \\
  \texttt{mzm0183@auburn.edu} \\
  \AND
  Ekene ~Aguegboh \\
  Department of Agricultural Economics and Rural Sociology\\
  Auburn University\\
  Auburn, AL 36849 \\
  \texttt{esa0013@auburn.edu} \\
  \And
  Hans-Werner ~van Wyk \\
  Department of Mathematics and Statistics\\
  Auburn University\\
  Auburn, AL 36849 \\
  \texttt{hzv0008@auburn.edu} \\
}

\begin{document}

\maketitle

\begin{abstract}
  The Adam optimization method has achieved remarkable success in addressing contemporary challenges in stochastic optimization. This method falls within the realm of adaptive sub-gradient techniques, yet the underlying geometric principles guiding its performance have remained shrouded in mystery, and have long confounded researchers. In this paper, we introduce GeoAdaLer (Geometric Adaptive Learner), a novel adaptive learning method for stochastic gradient descent optimization, which draws from the geometric properties of the optimization landscape. Beyond emerging as a formidable contender, the proposed method extends the concept of adaptive learning by introducing a geometrically inclined approach that enhances the interpretability and effectiveness in complex optimization scenarios.
\end{abstract}

\section{Introduction}
Stochastic gradient descent (SGD) optimization methods \citep{robbins1951stochastic, rumelhart1986learning} play an important role in various scientific fields. When applied to machine learning algorithms, the objective is to adjust a set of parameters with the goal of optimizing an objective function.  This optimization usually involves a series of iterative adjustments made to the parameters in each step as the algorithm progresses \citep{kingma2014adam}.  In the vanilla gradient descent approach, the magnitude of the gradient is the predominant annealing factor causing the algorithm to take larger steps when you are away from the optimum and smaller steps when closer to an optimum \citep{ruder_overview_2017}. This method, however, becomes less effective near an optimal point, necessitating the selection of a smaller learning rate, which in turn affects the speed of convergence. The raw magnitude of the gradient does not always align with the optimal descent step size, thus necessitating a manually chosen learning rate. If the learning rate is too big, overshooting may occur and convergence rate is slow if the learning rate is too small. This challenge has led to the development of adaptive learning algorithms \citep{nar2018step}. In this paper, we explore gradient descent algorithms that optimize the objective function with emphasis on the update rule. In this context, we focus on the update rule by breaking it into three components: the learning rate, the annealing factor, and the descent direction. This approach allows us to evaluate the effectiveness of an algorithm by examining the impact of its learning rate, annealing factor, and descent direction on the optimization process.

In this paper, we propose GeoAdaLer (short for Geometric Adaptive Learner), a new adaptive learning method for SGD optimization that is based on the geometric properties of the objective landscape. We use cosine of $\theta$ (where $\theta$ is the acute angle between the normal to the tangent hyperplane and the horizontal hyperplane) as an annealing factor, which takes values close to zero when the optimization is traversing points close to an optimum and close to one for points far away from an optimum. Our method has surprising similarities to other adaptive learning methods.

Some of the advantages of GeoAdaLer is that it introduces a geometric approach for the annealing factor and outperforms standard SGD optimization methods due to the cosine of $\theta$. Through both theoretical analysis and empirical observation, we identified similarities between our proposed method and other adaptive learning techniques. 
Our method enhances the understanding of existing algorithms and opens up more opportunities for geometric interpretability of how the algorithms traverse the objective manifold.

Additionally, we analyze the convergence of GeoAdaLer. We frame the optimization process as a fixed-point problem and split the analysis into deterministic and stochastic cases. Under the deterministic framework, we assume convexity and the existence of a finite optimal value. We utilize the Lipschitz continuity property of the gradient to establish upper bounds of convergence. We further employ the co-coercivity and quadratic upper bound properties to establish the lower bounds of convergence \citep{vandenberghe_optimization_nodate}. Under the stochastic framework, we employ the regret function, which measures the overall difference between our method and the known optimum point, ensuring that as time tends to infinity, the regret function over time tends to zero. By ensuring that the upper bound of the regret function goes to zero, we determine the overall convergence of the stochastic method to an optimum point. Both the deterministic and stochastic analyses demonstrate the robustness of GeoAdaLer and enhance our understanding of its practical applications.

\section{Related Work}

In the field of adaptive stochastic gradient descent algorithms, we continue to see improvements, often due to the need to address shortcomings of previous methods. The pioneering approach is AdaGrad, which focuses on the concept of per-parameter adaptive learning rates. The foundation of AdaGrad is also the source of its limitation: the monotonic accumulation of squared gradients that could prematurely stifle learning rates \citep{zeiler2012adadelta, duchi2011adaptive}.

Subsequent optimizers like AdaDelta, RMSprop, and the popular Adam sought to address this issue \citep{kingma2014adam,Tieleman2012rmsprop,zeiler2012adadelta}. AdaDelta introduced a decaying average of past squared gradients, while RMSprop utilized a similar exponential decay mechanism to limit the aggressive reduction in learning rates. Adam combined RMSprop's adaptive learning rates with the concept of momentum for smoother updates. However, Adam has been challenged by AMSGrad which modifies the algorithm in order to imbue it with "long term memory" which addresses issues with sub-optimal convergence under specific conditions and improves empirical performance \citep{reddi2019convergence}. As pointed out by \citep{li2024convergence, shi2021rmsprop, zhang2022adam}, the assumptions on the hyperparameters of Adam were made before constructing the counter examples in \citep{reddi2019convergence}. Insights gained from these examples are invaluable for the ongoing exploration of stochastic gradient descent algorithms.

The literature on adaptive optimization algorithms for SGD reveals that each algorithm addresses the problems evident in its predecessors. The iterative approach, while successful in many ways, arguably emphasizes the lack of fundamental intuition regarding the dynamics of the AdaGrad family of optimizers.

GeoAdaLer is a novel adaptive learning method for SGD, employing the properties of the objective landscape. The innovative idea behind the GeoAdaLer approach is that the acute angle between the normal to the tangent plane at $x$ and the horizontal plane conveys significant curvature-related information. This information could potentially recover the power of second order methods, which are not feasible in large-scale machine learning optimization. An understanding of this geometric approach to analyze adaptive methods for SGD can shed new light on the behavior of other algorithms in the AdaGrad family, potentially revealing the rationale for their strengths and weaknesses.

By discussing the geometric implications of adaptive step processes, we are able to potentially come up with optimizers that:

\begin{itemize} 
    \item \textbf{have more optimal annealing}: A geometric perspective that informs strategies to control learning rate decay more effectively, preventing premature convergence or extensively slow convergence.
    \item \textbf{provide parameter-sensitive adaptivity}: The geometry reveals how to tailor updates for individual parameters in a more principled manner.
    \item \textbf{increase robustness}: Understanding the geometric implications  leads to optimizers that are less dependent on sensitive hyper-parameter tuning.
\end{itemize}

\noindent
In essence, a geometric framework promises to move beyond the reactive development pattern, allowing us to proactively design adaptive optimizers that address the core issues in the AdaGrad family with greater intuition and foresight.

\section{Mathematical Formulation}
\subsection{Deterministic Optimization}
Consider the minimization of the convex objective functional $f: \mathbb{R}^n \to \mathbb{R}$ using the gradient descent algorithm.
The vanilla gradient descent (GD) algorithm that maximizes the benefit of \textit{gradient annealing} for smooth functions is
\begin{align}
x_{t+1} = x_t + \delta x_t
\end{align}
where $\delta x_t = -\gamma g_t$ and $g_t$ is $\nabla f(x_t)$, the gradient of the objective function at time $t$. This is the update step and is largely responsible for how far a step is taken in the descent direction. To see its real contribution, we decompose it further into
\begin{align}
\delta x_t = - \gamma |g_t|\overline{g_t}
\end{align}
where $\gamma$ is the learning rate which is responsible for manually scaling the step size, $|g_t|$ which is our annealing factor and lastly $-\overline{g_t}$ which gives us our optimal descent direction.

It is established that the annealing factor tends to be sub-optimal and can cause overshooting which we need to compensate for by applying a smaller learning rate, increasing convergence time \citep{nar2018step}.
To combat this issue, we propose an annealing factor based on the cosine of $\theta$, where $\theta$ is the acute angle between the normal to the tangent hyperplane and the horizontal hyperplane.
As shown in Figure \ref{theta-visualization}, $\theta$ holds a vital information about the location of the current gradient step on the objective function. We harness this information using the cosine of $\theta$ and prove the following theorem.

\begin{figure}[h!]
    \centering
    \begin{subfigure}{0.43\textwidth}
        \centering
    \includegraphics[width=\linewidth]{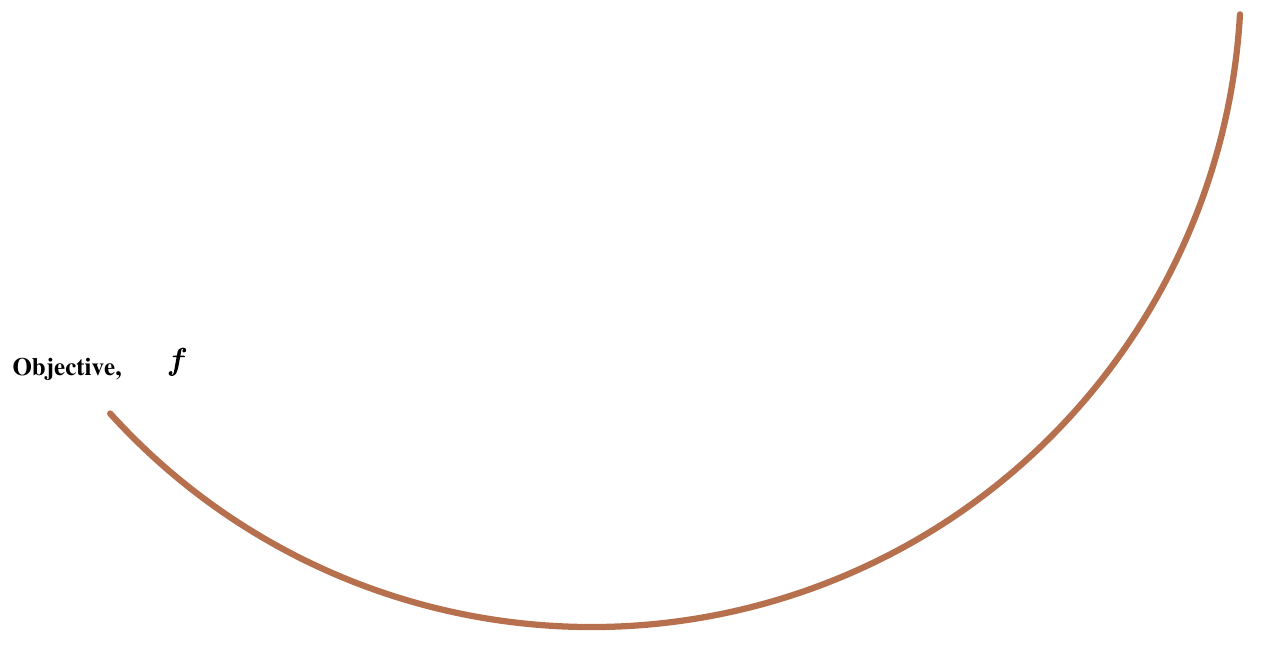}
        \caption{}
    \end{subfigure}
    \hspace{0.08\textwidth}
    \begin{subfigure}{0.45\textwidth}
        \centering
    \includegraphics[width=\linewidth]{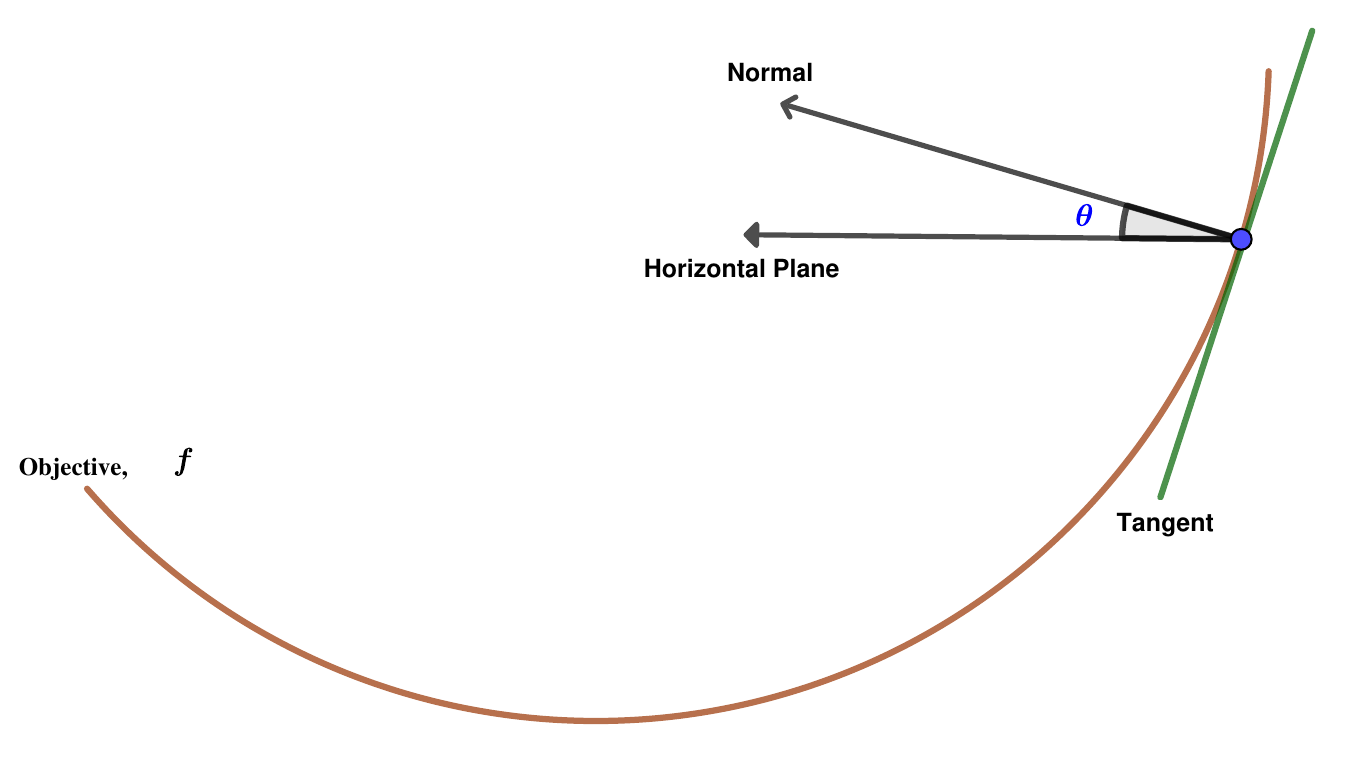}
        \caption{}
    \end{subfigure}
    \hspace{-0.04\textwidth}
    \begin{subfigure}{0.45\textwidth}
        \centering
    \includegraphics[width=\linewidth]{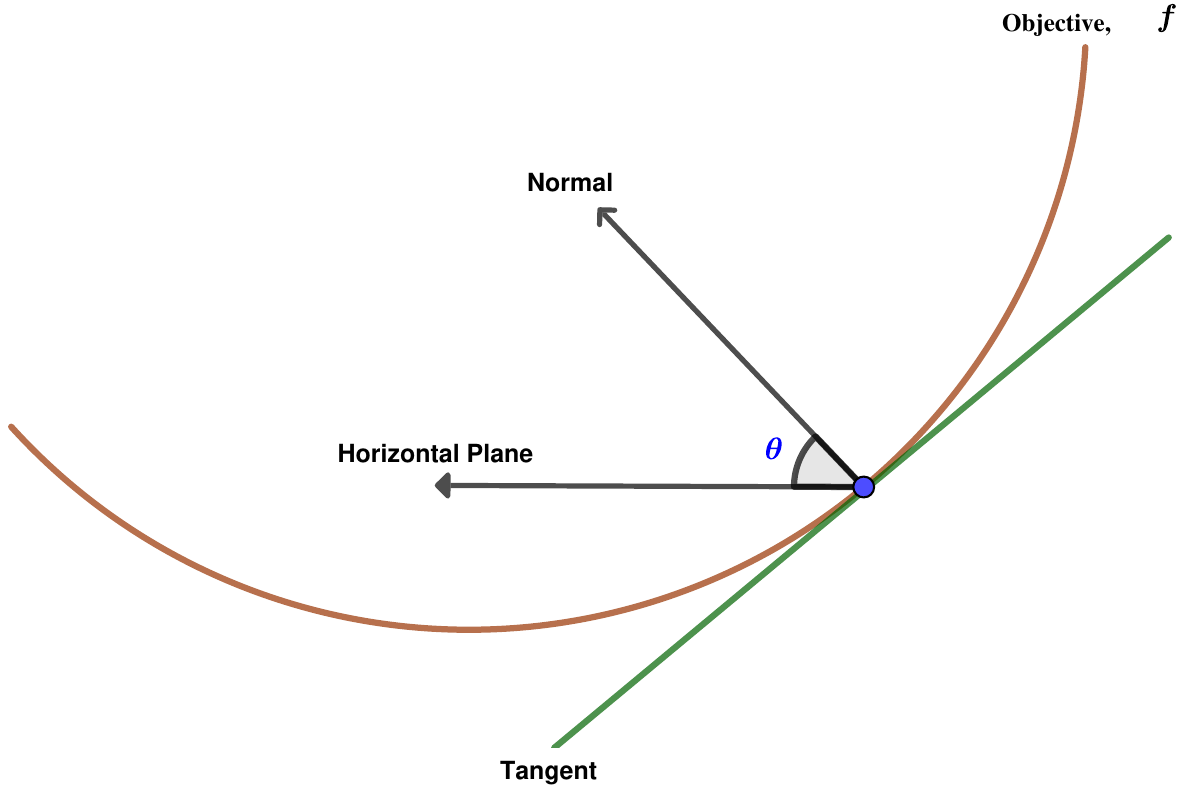}
        \caption{}
    \end{subfigure}
    \hspace{0.03\textwidth}
    \begin{subfigure}{0.45\textwidth}
        \centering
    \includegraphics[width=\linewidth]{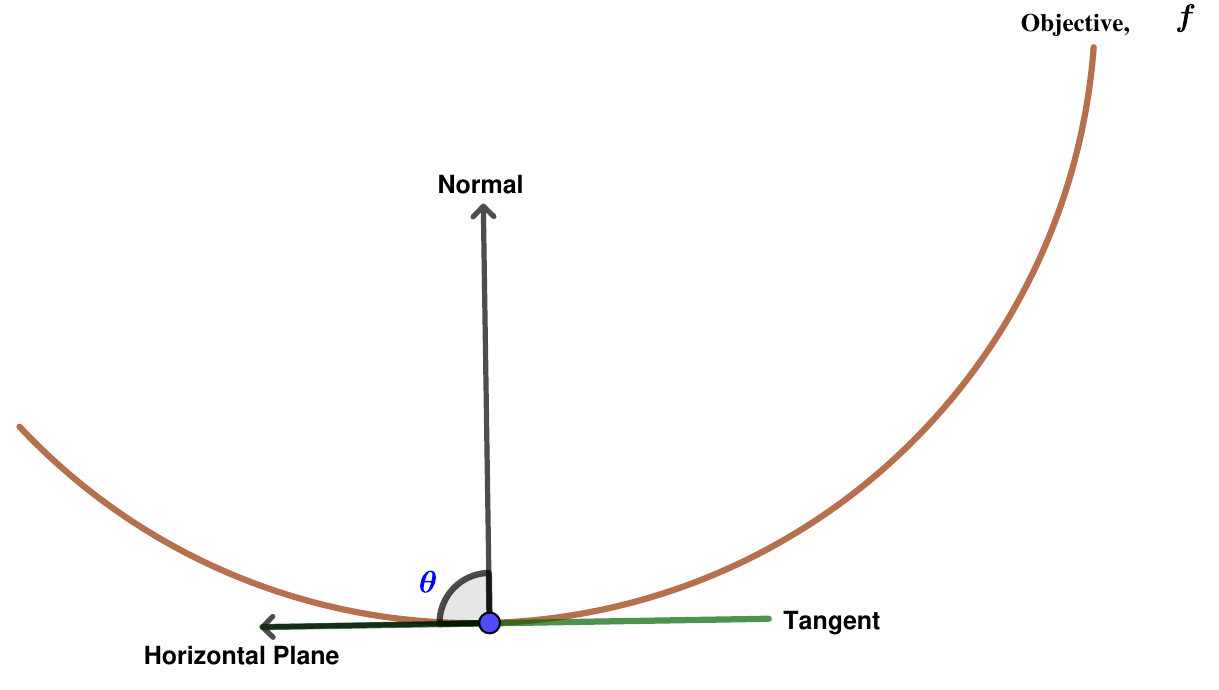}
        \caption{}
    \end{subfigure}
    \caption{Geometric view of $\theta$ as the GeoAdaLer step traverses the objective function.}
    \label{theta-visualization}
\end{figure}

\begin{theorem}[Geohess]\label{geohess}
Let $\theta$ be the acute angle between the normal to an objective function $f: \mathbb{R}^n \to \mathbb{R}$ which is differentiable at $x$. Let $\Vert \cdot \Vert$ be the norm induced by the inner product on $\mathbb{R}^n$.
Then 
\begin{align}
\cos \theta = \frac{\Vert \nabla f(x) \Vert}{\sqrt{\Vert \nabla f(x) \Vert^2 + 1} }
\end{align}
\end{theorem}
We call Theorem \ref{geohess} \textit{Geohess} since this formulation mimics the curvature information traditionally found in the full hessian matrix. \\

\begin{proof}
Let $x_i$ be an arbitrary point in $\mathbb{R}^n$. Then $[-\nabla f(x_i), 1]^T$ is orthogonal to $[x-x_i, y-f(x_i)]$ where $(x,y)$ lies on the tangent hyperplane to $f$ at $x_i$. i.e., $[-\nabla f(x_i), 1]^T$ is normal to $[x-x_i, y-f(x_i)]$. 
Let $\theta$ be the angle this normal makes with the horizontal hyperplane at the point $(x_i, f(x_i))$ in the descent direction $-\nabla f(x_i)$, i.e., a vector parallel to $[-\nabla f(x_i), 0]$. Then by vector calculus,
\begin{align}
\cos \theta &= \frac{[-\nabla f(x_i), 1]^T \cdot [-\nabla f(x_i), 0]^T}{
\Vert [-\nabla f(x_i), 1]^T\Vert \Vert [-\nabla f(x_i), 0]^T \Vert
}= \frac{\Vert \nabla f(x_i)\Vert}{\sqrt{\Vert \nabla f(x_i)\Vert^2 + 1} }
\end{align}
as required.    
\end{proof}

\noindent
The proposed GeoAdaLer update step is as follows:
\begin{align}\label{geo-update}
\delta x_t = - \gamma \, \overline{g_t} \cos\theta
 = - \gamma \frac{\Vert  g_t\Vert}{\sqrt{\Vert  g_t\Vert^2 + 1} }\overline{g_t}
\end{align}
Equation \eqref{geo-update} follows from the Geohess theorem (Theorem \ref{geohess}).
In equation \eqref{geo-update}, if $g_t$ is zero, then the update step is zero and algorithm has converged. If $g_t$ is not zero, then equation \eqref{geo-update} reduces to
\begin{align}
\delta x_t = - \gamma \frac{1}{\sqrt{\Vert  g_t\Vert^2 + 1} }g_t \nonumber
\end{align}

\subsection{Properties of GeoAdaLer Annealing Factor}
The acute angle $\theta$ depends on $g_t$,
and therefore, the annealing factor possesses the following properties:
\begin{multicols}{2}
\begin{enumerate}
\item $\displaystyle 
\frac{\Vert g_t \Vert}{\sqrt{\Vert g_t \Vert^2 + 1}} \to 1 \qquad \text{as} \quad g_t \to \infty \label{propostion-1}
$

\item $\displaystyle 
\frac{\Vert g_t \Vert}{\sqrt{\Vert g_t \Vert^2 + 1}} \to 0 \qquad \text{as} \quad g_t \to 0  \label{proposition-2}
$

\end{enumerate}
\end{multicols}

\noindent
In practice, we do not want $g_t \to \infty$ but as long as the magnitude of the gradient is large, the sufficiency in item \ref{propostion-1} is guaranteed. Of utmost importance is item \ref{proposition-2}.
In comparison to gradient annealing in SGD, we find that $\cos \theta$ possesses a logarithmic decay as a function of $g_t$ as it tends to an optimum, resulting in a much more controllable annealing while standard SGD  decays linearly.

The algorithm for GeoAdaLer is given in Algorithm \ref{geo-algorithm}. Note: The deterministic setting is achieved by using the objective function over the entire dataset instead of the stochastic objective function, in which case we set $\beta=0$.

\subsection{Stochastic Optimization}
Online learning and stochastic optimization are closely linked and can be essentially used interchangeably \citep{duchi2011adaptive, cesa2004generalization}. 
In online learning, the learner iteratively predicts a point $x_t \in X \subseteq \mathbb{R}^n$, often representing a weight vector that assigns importance values to different features. The objective of the learner is to minimize regret compared to a fixed predictor $x^*$ within the closed convex set $x_t \in X \subseteq \mathbb{R}^n$ across a sequence of functions $\{f_1, f_2,  \cdots   \}$.

\noindent
Geometrically, gradients are vectors and have directions. Under a suitable distribution, a convex hull of these gradients approximates the true gradient for stochastic optimizations.

A common and intuitive approach to approximating the expected gradient is by employing an exponential moving average (EMA). As is the practice in literature \citep{hinton2012neural, kingma2014adam}  we use a momentum-like term to replace the instantaneous gradient, thereby mitigating stochastic fluctuations and revealing underlying trends in the gradient values. That is,
\begin{align}
    m_{t+1}= \beta m_{t}+ (1-\beta) g_t.
\end{align}
where $\beta \in [0, 1)$.
The term $\beta$ is constructed as a weighted average of the historical gradients and the current gradient. This modification enhances our ability to approximate the true underlying gradient more effectively in time.
It also has the ability to update and learn as new streams of data are observed.

\noindent
Algorithm \ref{geo-algorithm} presents a pseudo-code of the proposed GeoAdaLer learning method. 

\begin{algorithm} 
\caption{GeoAdaLer}\label{geo-algorithm}
\begin{algorithmic}
\Require $\gamma$ : Learning rate
\Require $\beta$ : Exponential decay rate for weighted average
\Require $f_t$ : Stochastic objective function
\Require $x_0$ : Initial parameter vector
\State $t \gets 0$
\While{$x_t$ \text{not converged}}
\State $g_t \gets \nabla f_t(x_t)$ (Get gradients w.r.t stochastic objective)
\If {$t = 0$}
\State $m_t \gets g_t$ (initial weighted average)
\Else
\State $m_t=\beta m_{t-1} +(1-\beta)g_t$ (Updated weighted average)
\EndIf
\State$x_{t+1}=x_t-\gamma \cdot m_{t}/(\sqrt{\Vert m_{t}\Vert^2+1})$ (Update parameters)
\State $t \gets t+1$
\EndWhile\\
\Return $x_t$ 
\end{algorithmic}
\end{algorithm}

\subsubsection{GeoAdaMax}
Due to stochasticity and varying frequencies of occurrence of certain model inputs, such as in deep neural networks, adaptive stochastic gradient descent methods sometimes encounter issues with non-increasing squared gradients. 
 For instance, consider the convex objective function presented in \cite{reddi2019convergence}, defined as follows:
\[
f_t(x) = 
\begin{cases}
    Cx, & \text{if } t \mod 3 = 1\\
    -x, & \text{otherwise},
\end{cases}
\]
In such scenarios, the adaptive step employed by the optimizer can still function effectively as a form of annealing, but the vanilla adaptive step leads to a suboptimality.  

To address this problem, an approach was developed by \cite{reddi2019convergence}, which involves retaining the maximum of the normalizing denominator over iterations. This approach reduces to  a self scaling SGD with momentum
\citep{reddi2019convergence}. By implementing a similar idea, we observe related results for GeoAdaLer.  We call this method GeoAdaMax, indicating the use of maximum of the variance term.

GeoAdaMax dynamically adjusts the denominator using the largest historical value of the EMA. When this denominator remains unchanged, the update scale reflects its historical maximum, preserving proportionality in all future updates. This mechanism fine-tunes step sizes while maintaining alignment with past gradient magnitudes. The summary of the algorithm is presented in Algorithm \ref{geomax-algorithm}

Geometrically, this is akin to increasing the angle between the normal and the horizontal plane by using a vector different from the normal. This leads to relatively smaller step sizes than if the original angle were used. Theorem \ref{geoadamax-theta} shows that indeed, the acute angle $\theta$ is increased.

\begin{theorem}\label{geoadamax-theta}
Let $\theta$ be the acute angle between the normal and the horizontal hyperplanes at the current iteration and let $\hat{\theta}$ be the acute angle between the horizontal hyperplane and the normal that maximizes the norm up to the current iteration. Then
\begin{align*}
    \theta \le \hat{\theta}.
\end{align*}
\end{theorem}

\begin{proof}
 By Algorithm \ref{geo-algorithm} and Theorem \ref{geohess},
 \begin{align*}
\cos \theta = \frac{\Vert m_t \Vert}{\sqrt{\Vert m_t \Vert^2 + 1}} 
\ge \frac{\Vert m_t \Vert}{\sqrt{\underset{t}{\max} \Vert m_t \Vert^2 + 1}}
= \cos \hat{\theta}
 \end{align*}
 Applying the inverse cosine function on both sides gives the desired result since $\cos^{-1}$ is monotone decreasing on $[0,1]$.
\end{proof}

\begin{algorithm} 
\caption{GeoAdaMax}
\label{geomax-algorithm}
\begin{algorithmic}
\Require $\gamma$: Learning rate
\Require $\beta$: Exponential decay rate for weighted average
\Require $f_t$: Stochastic objective function
\Require $x_0$: Initial parameter vector
\State $t \gets 0$
\State $u_t \gets 0$
\While{$x_t$ \text{not converged}}
\State $g_t \gets \nabla f_t(x_t)$ (Get gradients w.r.t stochastic objective)
\If {$t = 0$}
\State $m_t \gets g_t$ (initial weighted average)
\Else
\State $m_t=\beta m_{t-1} +(1-\beta)g_t$ (Update weighted average)
\EndIf
\State $u_t=max(\Vert m_{t}\Vert^2+1,u_{t-1})$
\State $x_{t+1}=x_t-\gamma \cdot m_{t}/(\sqrt{u_t})$ (Update parameters)
\State $t \gets t+1$
\EndWhile\\
\Return $x_t$ 
\end{algorithmic}
\end{algorithm}

\section{Relationship to the AdaGrad Family}

The AdaGrad family of methods adjusts learning rates based on the accumulation of past gradients. These methods dynamically adapt the step size during optimization to handle varying gradient magnitudes across dimensions.
For a given time step $t$, the learning rate adjustment in these methods depends on the accumulated gradient information, which we represent as $G_t$. Below is a summary of how $G_t$ is defined for the main methods in the AdaGrad family:

\textbf{AdaGrad}: For AdaGrad, $G_t$ is given as
    \[
    G_t = \sum_{i=1}^t g_i^2
    \]
    where $g_i$ is the gradient at step $i$. AdaGrad accumulates the squared gradients over time, leading to decreasing learning rates \citep{duchi2011adaptive}.
    
\textbf{RMSProp}: The $G_t$ step in RMSProp involves exponential moving average and is given as
    \[
    G_t = \beta G_{t-1} + (1 - \beta) g_t^2.
    \]
    RMSProp introduces an exponential decay factor $\beta$, preventing the learning rate from decreasing too quickly by giving more weight to recent gradients \citep{hinton2012neural}.
    
AdaGrad and RMSProp modify the learning rate as follows:
\[
x_{t+1} = x_t - \frac{\gamma}{\sqrt{G_t + \epsilon}} \nabla f_t(x_t),
\]
where $\eta$ is the global learning rate and $\epsilon$ is a small constant for numerical stability.

\textbf{Adam}: In the case of Adam, an adjustment is not only made to the squared gradients but also to the gradients where a bias-corrected moving average is applied to each:
\begin{align*}
    m_t &= \beta_1 m_{t-1} + (1 - \beta_1) g_t\\
    G_t &= \beta_2 G_{t-1} + (1 - \beta_2) g_t^2
\end{align*} 
where $m_t$ and $G_t$ are the moving average of the gradients and the squared gradients respective, $\hat{m}_t = m_t/(1-\beta_1^t)$ and $\hat{G}_t = G_t/(1-\beta_2^t)$ are the bias-corrected moving average of gradients and  squared gradients respectively. Adam combines the ideas of momentum and adaptive learning rates for smoother updates \citep{kingma2014adam}.

The Adam's update rule therefore is as follows:
\[
x_{t+1} = x_t - \frac{\gamma}{\sqrt{\hat{G}_t} + \epsilon}\hat{m}_t,
\]

\subsection{GeoAdaLer in Comparison}
The core of GeoAdaLer’s update is based on the cosine of the angle $\theta$ between the normal to the gradient and the horizontal hyperplane. It effectively uses the geometry of the optimization landscape to inform its adaptivity. Hence, with an EMA update
\[
m_t = \beta m_{t-1} + (1 - \beta) g_t,
\]
GeoAdaLer updates as:
\[
x_{t+1} = x_t - \frac{\gamma}{\sqrt{\|m_t\|^2 + 1}} m_t.
\]

\noindent
As shown above, GeoAdaLer's update is similar to that of the AdaGrad family with some function of the squared gradient playing a big role in the optimization step. The main differences include:

\begin{enumerate}[label=(\alph*)]
    \item The use of norm-based scaler
    \item The stability term in GeoAdaLer is naturally derived from the choice of the reference plane, for example with the choice of the normal to the gradient, our scaler produces a stability term of 1. 
    \item There is only one moving average which is used in the estimation of the gradient and the function of the squared gradients
    \item Directly comparing to Adam, we also agree that $G_t$ is always bigger for Adam  since by Jensen's inequality,
    $$
    \left[ \beta m_{t-1} + (1-\beta) g_t \right]^2 \le \beta m_{t-1}^2 + (1-\beta) g_t^2,
    $$
    where the LHS $G_t$ is for GeoAdaLer and  the RHS for Adam.
\end{enumerate} 

\noindent
We remark that all of these differences stem from the geometric intuition behind GeoAdaLer, which we believe lends itself to better understanding of the optimization paths as well as interpreting optimal values. Also, as noted in \cite{ward2020adagrad}, the use of norm-based adaptivity ensures GeoAdaLer is robust to its hyperparameters.

\section{Convergence Analysis}
\subsection{Deterministic Setting}
We analyze the convergence of GeoAdaLer, first for the deterministic case, and then for the stochastic setting. Our setup remains the same, namely:
\begin{align}\label{optim-conv}
    \underset{x}{\text{minimize}} \, f(x)
\end{align}
using the new adaptive gradient descent (GeoAdaLer)  method where $f:\mathbb{R}^n  \to \mathbb{R}$ is an objective function. We recast the optimization problem \eqref{optim-conv} into a fixed point iteration. Let $T:\mathbb{R}^n \to \mathbb{R}^n$ be a nonlinear operator defined as:

\begin{align}\label{operator}
  Tx  &= \left( I - \gamma\frac{\nabla f}{\sqrt{\Vert \nabla f \Vert^2 + 1} } \right) (x)
\end{align}
where $I$ is the identity map.

\begin{theorem}\label{geo-deterministic}
Let $f:\mathbb{R}^n \to \mathbb{R}$ be continuous and $\nabla f$ Lipschitz continuous with $\gamma \le \frac{1}{L} $ where $L$ is the Lipschitz constant for $\nabla f$. Assume $f$ attains an optimal value at $x^* = \arg\min_x f(x)$. Then $T$ defined in \eqref{operator} is a contraction map with contraction parameter $\alpha = \sqrt{1 + \gamma^2 L_G^2 - 2\gamma \frac{L_G^2}{L}} < 1$ where $L_G$ is the Lipschitz constant for 
$\displaystyle \frac{\nabla f}{\sqrt{\Vert \nabla f \Vert^2 + 1}} $ and $\Vert \cdot \Vert$ is the Euclidean norm in $\mathbb{R}^n$. That is
\begin{align}
\Vert Tx - Ty \Vert \le \alpha \Vert x - y \Vert.
\end{align}
\end{theorem}

\noindent
A critical component of Theorem \ref{geo-deterministic} involves demonstrating that the mapping $T$ is a contraction. This allows us to invoke the Banach Fixed Point Theorem, which asserts  that any contraction mapping on a complete metric space possesses a unique fixed point \citep{chidume2009geo, rudin1976principles, sutherland2009introduction, granas2003fixed}. By iteratively applying the contraction map, starting from an initial point, we ensure convergence to this fixed point. The recursive iterations converge to the unique fixed point with a geometric rate of convergence $\alpha$. This fixed point corresponds to the minimizer we are seeking \citep{boyd2016primer}. 
We remark that Theorem \ref{geo-deterministic} holds if we replace continuity of $f$ with lower semicontinuity and gradient with subgradient.

\subsection{Stochastic Setting}
We examine the convergence properties of the GeoAdaLer optimizer within the framework of online learning, as originally introduced in \cite{zinkevich2003online}. This framework involves a sequence of convex cost functions $ \{f_1, f_2, \ldots, f_T\} $, each of which becomes known only at its respective timestep. The objective at each step $ t $ is to estimate the parameter $ x_t $ and evaluate it using the newly revealed cost function $ f_t $. Given the unpredictable nature of the sequence, we assess the performance of our algorithm by computing the regret. Regret is defined as the cumulative sum of the differences between the online predictions $ f_t(x_t) $ and the optimal fixed parameter $ f_t(x^*) $ within a feasible set $ X $ across all previous time-steps. Specifically, the regret is defined as follows \citep{zinkevich2003online,kingma2014adam}:
\begin{align}
R(T) = \sum_{t=1}^T \left(f_t(x_t) - f_t(x^*) \right)
\end{align}
where the optimal parameter $x^*$ is determined by $x^* = \arg\min_{x \in X} \sum_{t=1}^T f_t(x)$. We demonstrate that GeoAdaLer achieves a regret bound of $O(\sqrt{T})$, with a detailed proof provided in the appendix \ref{convergence-proof}. This result aligns with the best known bounds for the general convex online learning problem. We carry over all notations from the deterministic setting. Assuming the learning rate $\gamma_t$ is of order $O(t^{-1/2})$ and $\beta_t$ is exponentially decaying with exponential constant $\lambda$ very close to $1$, we obtain the following regret bounds for online learning with GeoAdaLer algorithm.

\begin{theorem}\label{geo-convergence}
For all \( x \in \mathbb{R}^n \), and $t \le T$, assume the gradient norm \(\Vert \nabla f_t(x)\Vert \leq G\) . Let \(\gamma_t = \frac{\gamma}{\sqrt{t}}\), \(\beta_t = \beta \lambda^t\), \(\lambda \in (0,1)\), and \(\beta \in [0,1)\). For any \( k \in \{1, \ldots, T\} \), the separation between any point \( x_k \) generated by GeoAdaLer and the minimizer of an offline objective computed after all data is known is bounded as \(\Vert x_k - x^*\Vert \leq G\). Then, for any \( T \geq 1 \), GeoAdaLer Algorithm achieves the regret bound:
\begin{align}\label{geo-bound}
R(T) \leq \frac{D^2 \sqrt{G^2 +1} \sqrt{T} + G  (2\sqrt{T}-1)}{2(1-\beta)} + \frac{D G \beta(1-\lambda^T)}{(1-\beta)(1-\lambda)}
\end{align}
\end{theorem}

\noindent
From Theorem \ref{geo-convergence}, we observe that the GeoAdaLer algorithm achieves a sublinear regret bound of \(O(\sqrt{T})\) over \( T \) iterations, which is consistent with the results typically found in the literature for similar algorithms. Our proof, akin to the approach taken in the Adam algorithm \citep{kingma2014adam}, relies significantly on the decay of $\beta_t$ to ensure convergence.

In contrast to other adaptive stochastic gradient descent methods that adapts the current employs norm based scaling. This approach ensures that each parameter benefits from the collective dynamics of all parameters at the current time while retaining historical information in the exponential moving average of the gradients used for the adaptation.
We remark that the convergence analysis for GeoAdaMax follows trivially from Theorem \ref{geo-convergence}.

\begin{corollary}
Assume that for any \(x \in \mathbb{R}^n\), the function \(f_t\) is convex and satisfies the gradient bounds \(\|\nabla f_t(x)\|_2 \leq G\). Also, assume that the distance between any parameter \(x_k\) generated by GeoAdaLer Algorithm and the minimizer of an offline objective computed after all data is known is bounded, namely \(\|x_k - x^*\|_2 \leq D\) for any \(k \in \{1, \cdots, T\}\). Then, GeoAdaLer achieves the following guarantee for all \(T \geq 1\):

\[
\limsup_{T \to \infty} \frac{R(T)}{T} \le 0
\]
\end{corollary}
\noindent
The corollary is derived by dividing the result in Theorem \ref{geo-convergence} by $T$ and applying the limit superior operation to both sides of the inequality \ref{geo-bound}. It is important to note that when \(R(T)\) yields a negative value, it indicates a favorable performance of the iterates produced by the GeoAdaLer algorithm. Specifically, such results suggest that the algorithm's execution leads to an expected loss that is lower than that of the best possible offline algorithm, which has full foresight of all cost functions and selects for a single optimal vector as proposed in \cite{zinkevich2003online}.

\section{Experiments}\label{experiments}

We compare GeoAdaLer to other algorithms such as Adam, AMSGrad, and SGD with momentum. This comparison aims to assess how the geometric approach performs against these popular methods on the standard CIFAR-10 and MNIST datasets. Additionally, we examine how key hyperparameters influence the convergence of GeoAdaLer across different datasets. 

The hyperparameter settings for the algorithms are detailed as follows. These settings represent default values unless otherwise specified as recommended in the literature or considered standard for their respective packages. For SGD, the learning rate was set to $0.01$, momentum to $0.9$ and dampening to $0.9$. For Adam \& AMSGrad, the learning rate was set to $0.001$, $\beta_1$ to $0.9$ and $\beta_2$ to $0.99$. All CPU calculations were performed on an AMD Ryzen 8-core CPU, and GPU calculations were conducted on a NVIDIA 3080ti.

\subsection{MNIST Dataset}
 
In the MNIST \citep{lecun1998} experiment we trained a fully connected feed forward neural networks . It consists of three fully connected layers: the first layer takes in 784 input features (flattened 28x28 grayscale images) and outputs 128 features, the second layer reduces these to 64 features, and the final layer outputs 10 logits corresponding to class scores. Each of the first two layers is followed by a ReLU activation function. The final layer provides raw class scores. Using cross-entropy loss, GeoAdaLer and GeoAdaMax algorithms alongside the baseline optimizers were executed for 150 epochs and for 30 different weights initializations. Figure \ref{fig1} and Table \ref{MNIST_table} illustrate the averaged results on the test dataset. GeoAdaLer demonstrates comparable initial performance to algorithms such as Adam and SGD, yet it achieves better long-run performance, converging to a higher validation accuracy. GeoAdaMax further improved this performance by faster convergence. 

\begin{figure} [h!]
    \centering
    \begin{subfigure}{0.46\textwidth}
        \includegraphics[width=\textwidth]{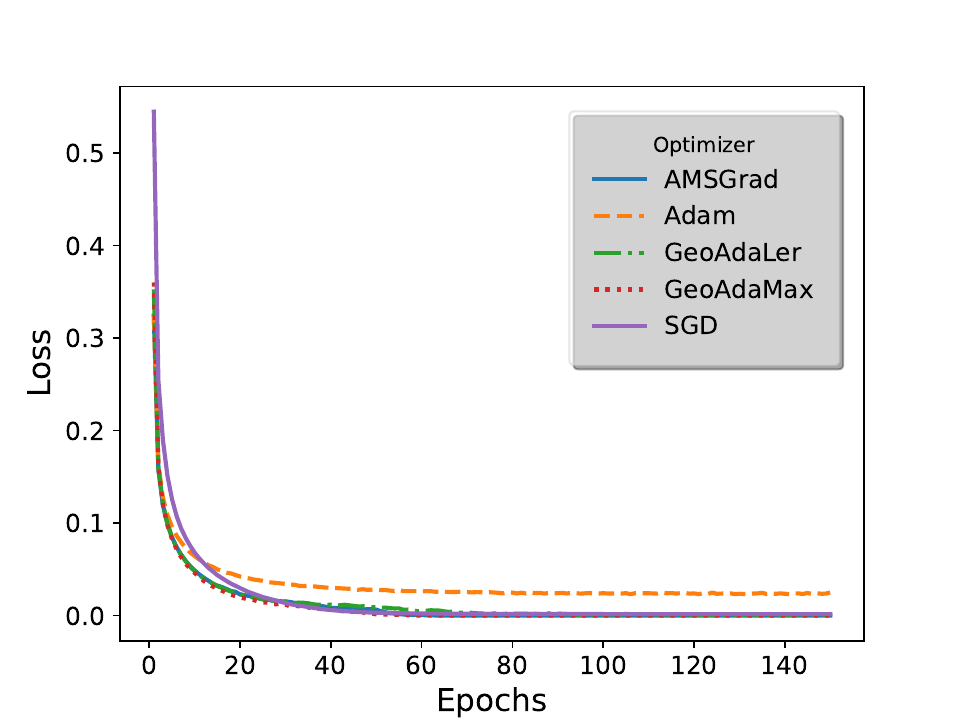}
        \caption{Training loss}
    \end{subfigure}
    \begin{subfigure}{0.46\textwidth}
        \includegraphics[width=\textwidth]{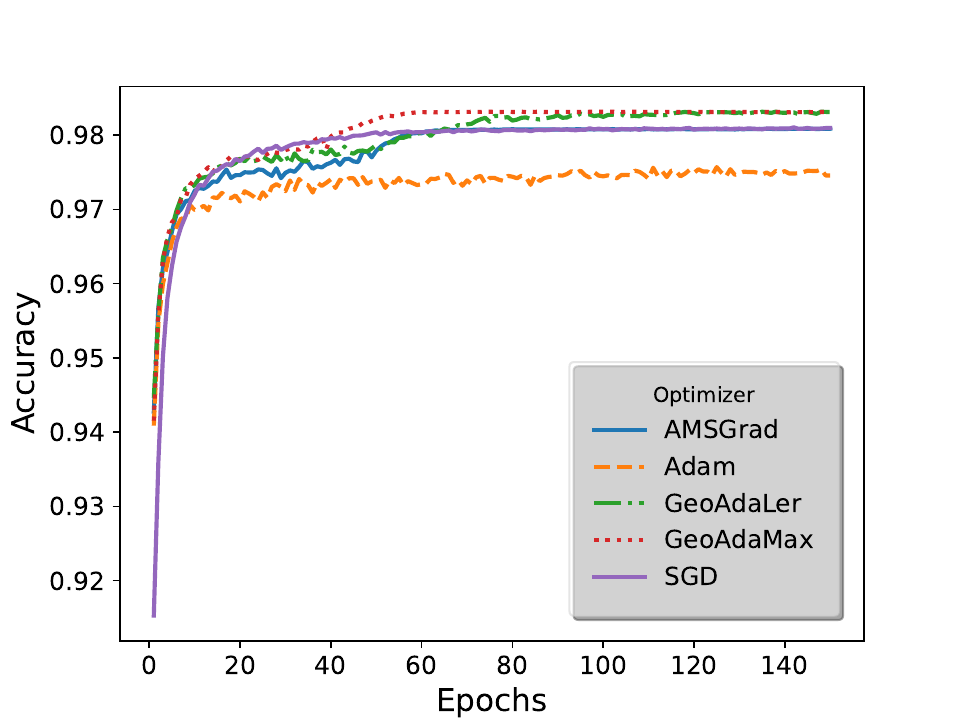}
        \caption{Validation Accuracy}
    \end{subfigure}
    \caption{MNIST}
    \label{fig1}
\end{figure}

\subsection{CIFAR-10 Dataset}

In our CIFAR-10 \citep{Krizhevsky2012} experiments, we run GeoAdaLer, GoeAdaMax and the baseline optimizers for 50 epochs on a network consisting of six convolutional layers with 3x3 kernels and padding, progressively increasing in sizes of $32,32,64,64,128$ and $128$ filters. Each convolution is followed by a batch normalization layer and a ReLU activation. After every two convolutional layers, max-pooling with a $2\times2$ kernel is applied to reduce spatial dimensions, followed by dropout layers with rates $0.2,0.3$ and $0.4$. The output from the convolutional block, consisting of 128 filters with a 4x4 spatial size, is flattened and passed to two fully connected layers: the first reduces the feature size to 128 with a ReLU activation and a 0.5 dropout, and the second outputs 10 logits corresponding to class scores which are passed to a softmax function. The model is trained using cross-entropy loss and re-run 30 times for each optimizer with different initializations of the weights. The averaged results on test data are shown in Figure \ref{fig2} and Table \ref{CIFAR_table}. 
GeoAdaLer shows early run performance comparable to baseline optimizers and occasionally exceeds them in validation accuracy.
Its long run performance was only matched by Adam and GeoAdaMax. The consistent performance of GeoAdaLer in various runs reinforces the value of incorporating a geometric perspective into its design. 

In the CIFAR-10 \citep{Krizhevsky2012} experiment, we trained a convolutional neural network (CNN) model designed specifically for image data with RGB channels. The architecture consists of six convolutional layers, progressively increasing in feature map depth (32, 64, and 128 channels), followed by max-pooling and dropout layers to reduce overfitting and improve generalization. Each convolutional layer is followed by batch normalization to stabilize learning and expedite convergence \citep{ioffe2015batch}. Finally, two fully connected (linear) layers map the feature space to the 10 CIFAR-10 class scores, following common practices for classification in CNN architectures \citep{krizhevsky2012imagenet}.

We train the model using cross-entropy loss for multi-class classification \citep{goodfellow2016deep}, and all adaptive gradient descent algorithms (including GeoAdaLer and GeoAdaMax) were executed for $100$ epochs. To account for random initializations, we average the losses and accuracies over $30$ different initializations. Model results are illustrated in Figure \ref{fig2} and Table \ref{CIFAR_table}. GeoAdaLer demonstrates initial performance on par with other adaptive optimizers, such as Adam \citep{kingma2014adam} and SGD \citep{lecun1998gradient}, but achieved superior long-term accuracy, converging to higher validation accuracy. GeoAdaMax provided further benefits, resulting in faster convergence and smoother training dynamics due to its stability near the optimal point.

\begin{figure}[h!]
    \centering
    \begin{subfigure}{0.46\textwidth}
        \includegraphics[width=\textwidth]{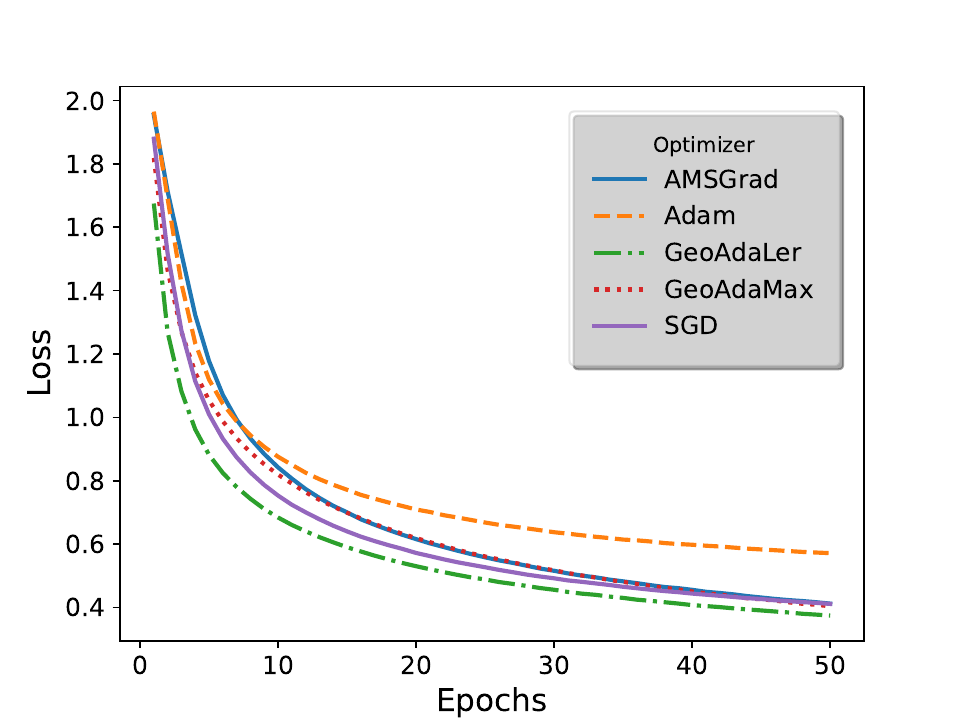}
        \caption{Training loss}
    \end{subfigure}
    \begin{subfigure}{0.46\textwidth}
        \includegraphics[width=\textwidth]{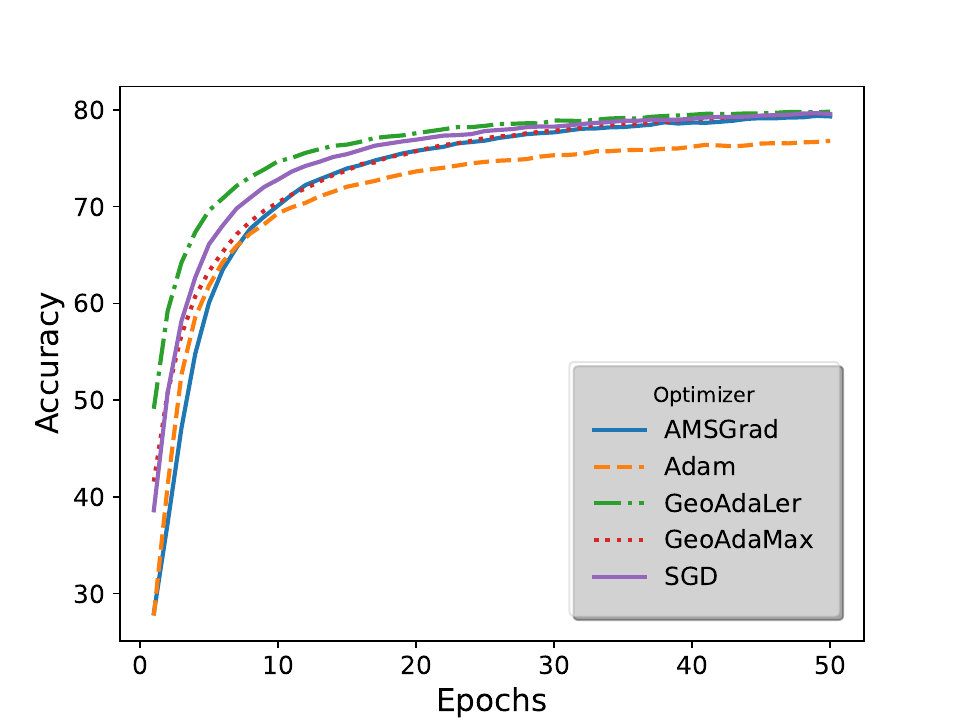}
        \caption{Validation Accuracy}
    \end{subfigure}
    \caption{CIFAR 10 Dataset}
    \label{fig2}
\end{figure}

\begin{multicols}{2}
\begin{table}[H]
\caption{MNIST Final Accuracy}
\label{MNIST_table}
\centering
\begin{tabular}{ll}
\toprule
Optimizer &  Accuracy \\
\midrule
GeoAdaLer & 0.9831\\
GeoAdaMax & 0.9831\\
Adam & 0.9746\\
AMSGrad & 0.9809\\
SGD & 0.9810\\
\bottomrule
\end{tabular}
\end{table}

\begin{table}[H]
\caption{CIFAR Final Accuracy}
\label{CIFAR_table}
\centering
\begin{tabular}{ll}
\toprule
Optimizer &  Accuracy \\
\midrule
GeoAdaLer & 0.7982\\
GeoAdaMax & 0.7962\\
Adam & 0.7679\\
AMSGrad & 0.7932\\
SGD & 0.7957\\
\bottomrule
\end{tabular}
\end{table}

\begin{table}[H]
\caption{Fashion MNIST Final Accuracy}
\label{FashionMNIST_table}
\centering
\begin{tabular}{ll}
\toprule
Optimizer &  Accuracy \\
\midrule
GeoAdaLer & 0.9044\\
GeoAdaMax & 0.9042\\
Adam & 0.8838\\
AMSGrad & 0.8993\\
SGD & 0.8969\\
\bottomrule
\end{tabular}
\end{table}
\end{multicols}

\subsection{Fashion MNIST}
In the Fashion MNIST \cite{xiao2017fashion} experiment we train a fully connected feed forward neural networks. It consists of three fully connected layers: the first layer takes in $784$ input features (flattened $28\times28$ grayscale images) and outputs $512$ features, the second layer reduces these to $256$ features, and the final layer outputs 10 logits corresponding to class scores. Each of the first two layers is followed by a ReLU activation function. The final layer provides raw class scores. Using cross-entropy loss, GeoAdaLer and GeoAdaMax algorithms alongside the baseline optimizers were executed for $100$ epochs and for $30$ different weights initializations. Figure \ref{fig3} and Table \ref{FashionMNIST_table} illustrate the averaged results on the test dataset. GeoAdaLer once again shows comparable results to that demonstrated by the other benchmark algorithms with GeoAdaMax showing further improvement and faster convergence on the Fashion MNIST dataset.

\begin{figure}[h!]
    \centering
    \begin{subfigure}{0.46\textwidth}
        \includegraphics[width=\textwidth]{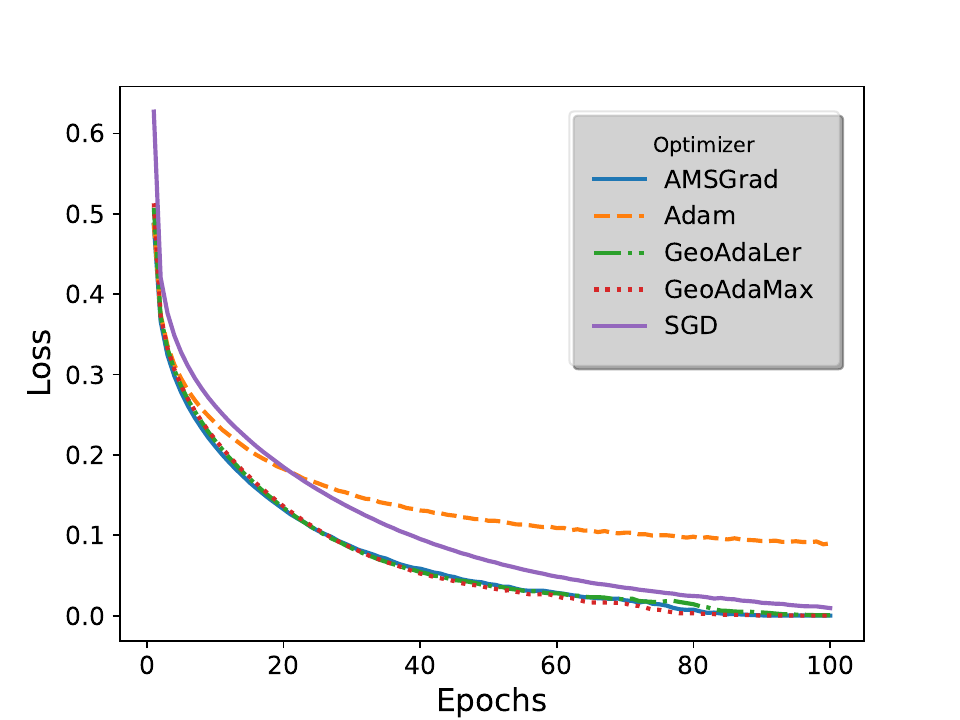}
        \caption{Training loss}
    \end{subfigure}
    \begin{subfigure}{0.46\textwidth}
        \includegraphics[width=\textwidth]{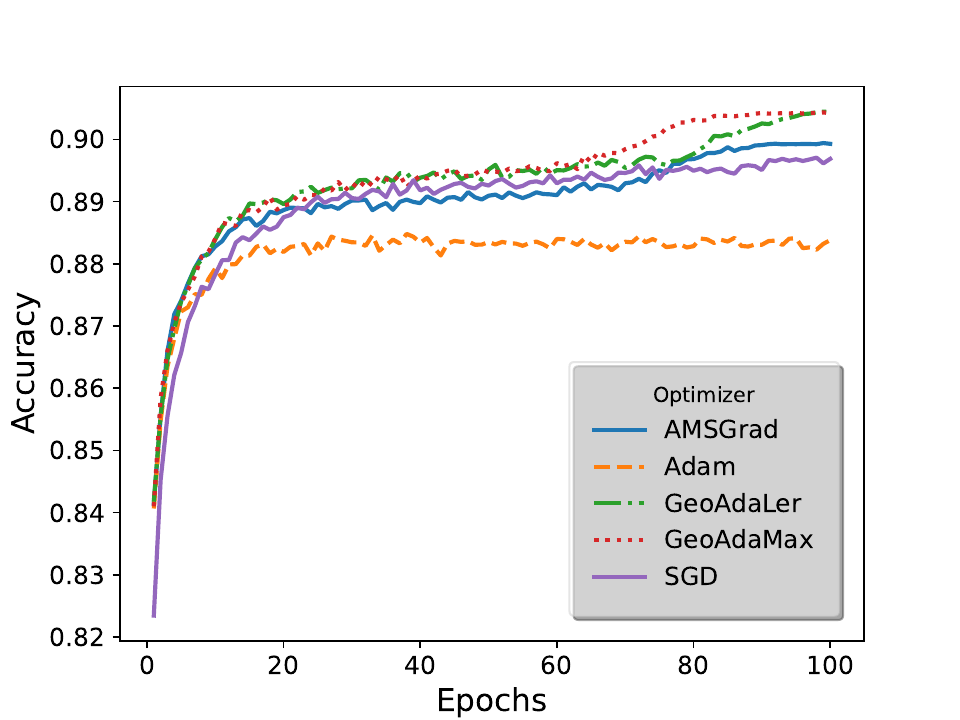}
        \caption{Validation Accuracy}
    \end{subfigure}
    \caption{FashionMNIST Dataset}
    \label{fig3}
\end{figure}

\subsection{Alternative Normal Plane Vectors}

By introducing a hyper-parameter $\epsilon$ in the update rule, we can explore different vectors within the normal plane:
\[
x_{t+1} = x_t - \frac{\gamma}{\sqrt{\|m_t\|^2 + \epsilon}} m_t.
\]
This approach allows us to select vectors with varying angles relative to the horizontal plane. This investigation stems from observing how GeoAdaMax modifies the effective angles by maximizing the denominator of the update.

In this experiment, we compared GeoAdaLer and GeoAdaMax on the MNIST (\ref{fig4}) and CIFAR-10 (\ref{fig5}) datasets using $30$ different weight initializations for selected values of $\epsilon$, all other parameters were as mentioned above in their individual experiments. The mean accuracies for each $\epsilon$ value are plotted versus the value of $\epsilon$. The results indicate that while the normal vector may not always be the most optimal choice, the optimal $\epsilon$ value tends to be close to the default associated with the normal vector. although some improvement do exist through the use of larger values of epsilon it is dependent on the data and not constant through all the experiments.

\begin{figure}[h!]
    \centering
    \begin{subfigure}{0.46\textwidth}
        \includegraphics[width=\textwidth]{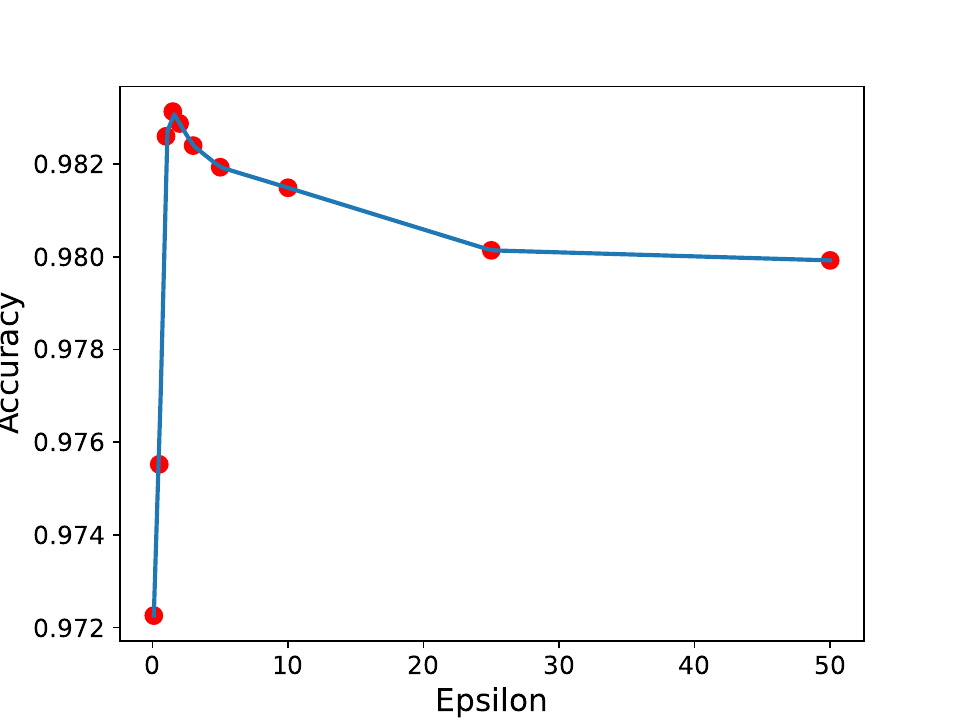}
        \caption{GeoAdaLer Accuracy}
    \end{subfigure}
    \begin{subfigure}{0.46\textwidth}
        \includegraphics[width=\textwidth]{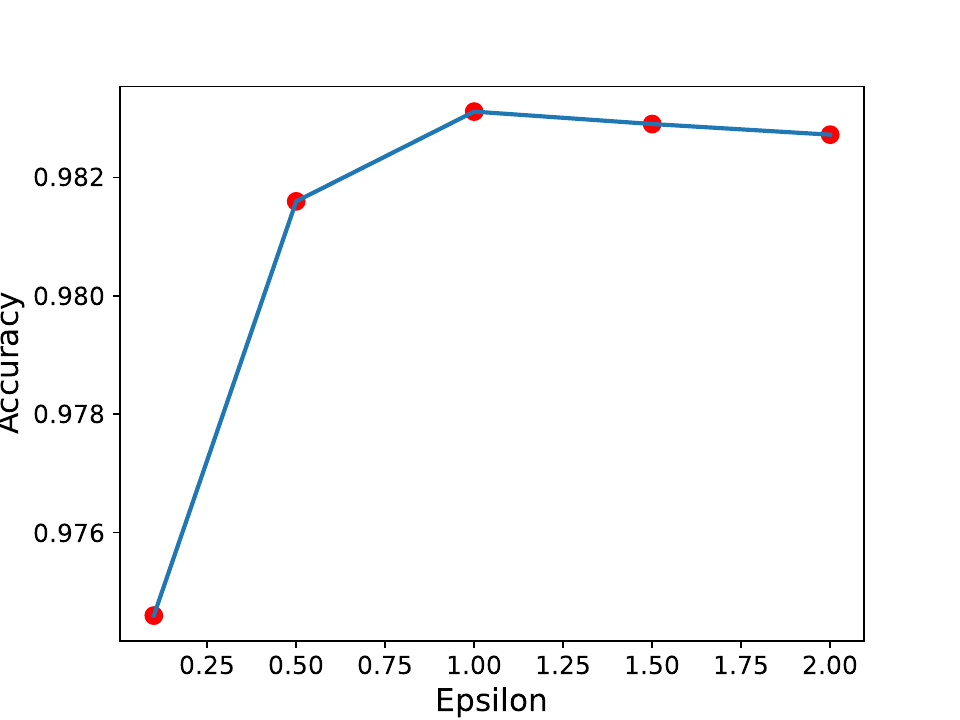}
        \caption{GeoAdaMax Accuracy}
    \end{subfigure}
    \caption{MNIST Dataset}
    \label{fig4}
\end{figure}

\begin{figure}[h!]
    \centering
    \begin{subfigure}{0.46\textwidth}
        \includegraphics[width=\textwidth]{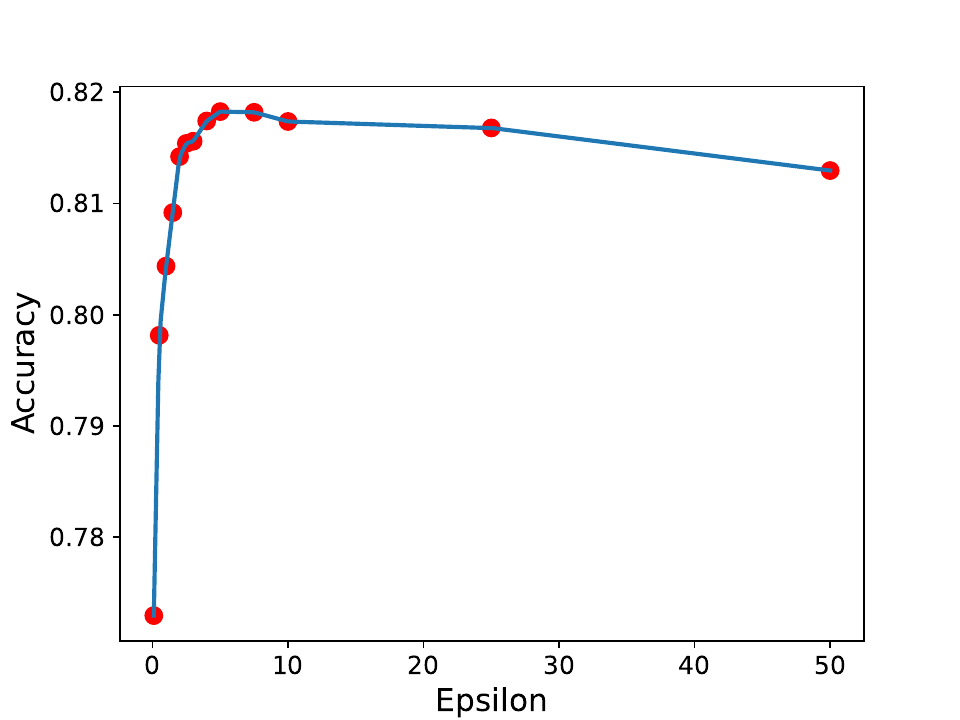}
        \caption{GeoAdaLer Accuracy}
    \end{subfigure}
    \begin{subfigure}{0.46\textwidth}
        \includegraphics[width=\textwidth]{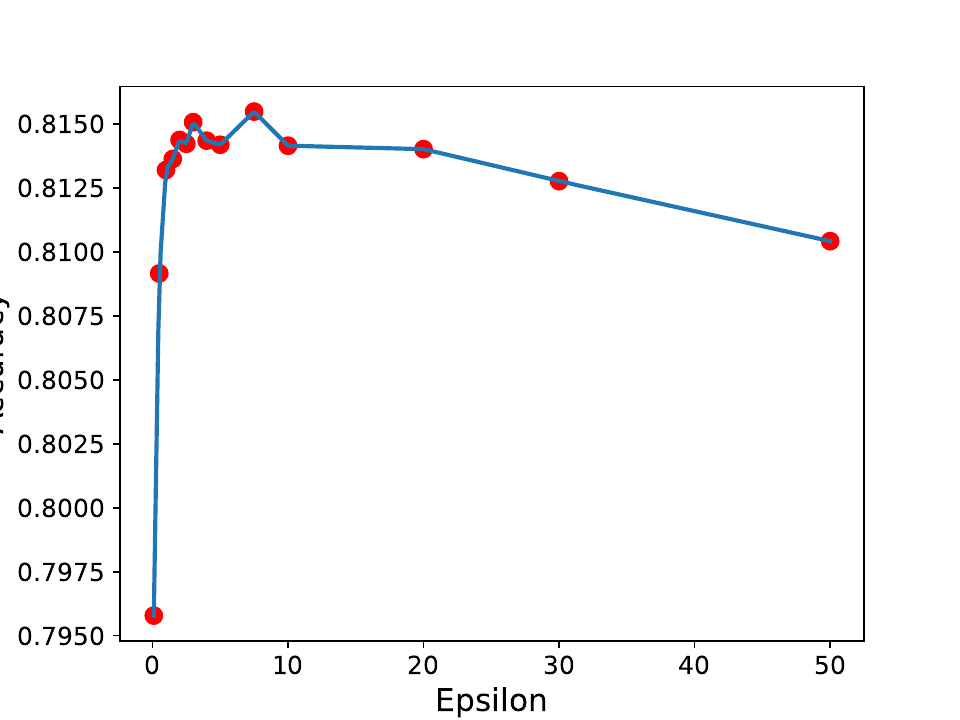}
        \caption{GeoAdaMax Accuracy}
    \end{subfigure}
    \caption{CIFAR 10 Dataset}
    \label{fig5}
\end{figure}

\section{Conclusion}

In this paper, we investigate the adaptive stochastic gradient descent algorithm and propose a geometric approach where the normal vector to the tangent hyperplane plays a crucial role in providing curvature-like information. We call this approach GeoAdaLer, and we show that it is derived from a fundamental understanding of optimization geometry. We present theoretical proof for both deterministic and stochastic settings.
Empirically, we find that GeoAdaLer is competitively comparable with other optimization techniques. Under certain conditions, it offers better performance and stability. GeoAdaLer provides a general geometric framework applicable to most, if not all, large-scale adaptive gradient-based optimization methods. We believe that this presents a significant step towards the development of interpretable machine learning algorithms through the lens of optimization.

\section*{Data and Code Availability}
Our codebase and the associated datasets used in our experiments are available in an open-source repository on GitHub: \hyperlink{https://github.com/Masuzyo/Geoadaler}{https://github.com/Masuzyo/Geoadaler}.

\vskip 0.2in
\bibliographystyle{plain}
\bibliography{ref.bib}

\newpage
\appendix
\section{Deterministic Convergence Proof}\label{proof-deterministic}
\begin{definition}
A differentiable function $f$ on $\mathbb{R}^n$ is said to be co-coercive if 
\begin{align}
\frac{1 }{L}\Vert\nabla f(x)-\nabla f(y) \Vert^2\leq \langle \nabla f(x)-\nabla f (y),x-y\rangle, \text { for all } x,y \in \mathbb R^n.
\end{align}
\end{definition}

\begin{lemma}
Let $\nabla f$ be Lipschitz continuous with constant $L>0$. Then $\frac{\nabla f}{\sqrt{\Vert \nabla f \Vert^2 + 1}}$ is Lipschitz continuous.
\end{lemma}
\begin{proof}
    Follows since $\frac{\nabla f}{\sqrt{\Vert \nabla f \Vert^2 + 1}}$ is a continuous function of $\nabla f$.
\end{proof}

\begin{lemma}\label{conseq-quadratic-bd}
    Suppose $\nabla f$ is Lipschitz continuous with parameter $L$,  domain of $f$ is $\mathbb R^n$ and $f$ has a minimum at $x^*$. Then
    $$\frac{1}{2L}\Vert\nabla f(z)\Vert^2\leq f(z)-f(x^*)$$
\end{lemma}
\begin{proof}
The proof relies on the following quadratic upper bound property
\begin{align}\label{quadratic-upper-bd}
f(y) &\leq f(z)+\langle \nabla f(z),y-z\rangle+\frac{L}{2}\Vert y-z\Vert^2 \quad \text{for all } y, z \in \text{ dom}(f)
\end{align}
where dom$(f)$. the domain of $f$ is a convex set. i.e., taking infimum on both sides of \ref{quadratic-upper-bd} gives
\begin{align}
        f(x^*)=\inf_y f(y) &\leq \inf_y(f(z)+\langle \nabla f(z),y-z\rangle+\frac{L}{2}\Vert y-z\Vert^2\nonumber\\ 
        &=\inf_{\Vert v \Vert=1}\inf_t \left (f(z)+t \langle \nabla f(z), v\rangle+\frac{Lt^2}{2} \right)\nonumber\\
        &=\inf_{\Vert v \Vert=1}\left (f(z)-\frac{1}{2L} \langle \nabla f(z), v\rangle \right)\nonumber\\
        &= f(z)-\frac{1}{2L}  \Vert\nabla f(z)\Vert^2 \label{conseq-quadratic-bd-pf}
\end{align}
Rearranging the terms in \ref{conseq-quadratic-bd-pf} gives the desired result.
\end{proof}

\begin{lemma}[Co-coercivity]\label{cocoer}
Assume $f$ is convex, proper and lower semicontinuous.
Let $\nabla f$ be Lipschitz and let $L_G$ be the Lipschitz constant for 
$\displaystyle \frac{\nabla f}{\sqrt{\Vert \nabla f \Vert^2 + 1}}$. 
Then the following co-coercivity property holds:
\begin{align}
\frac{1}{L_G}\left\Vert \frac{\nabla f(x)}{\sqrt{\Vert\nabla f(x)\Vert^2 +1}} - \frac{\nabla f(y)}{\sqrt{\Vert\nabla f(y) \Vert^2 +1}} \right\Vert^2 \leq \bigg \langle \frac{\nabla f(x)}{\sqrt{\Vert\nabla f(x)\Vert^2 +1}}- \frac{\nabla f(y)}{\sqrt{\Vert\nabla f(y)\Vert^2 +1}},x-y \bigg\rangle\
\end{align}
\end{lemma}

\begin{proof}
For any $x$ and $y$, let
\begin{align}
    f_x(z) \coloneqq F(z) - \frac{\nabla f(x)^Tz}{\sqrt{\Vert\nabla f(x)\Vert^2 +1}}, \\
    f_y(z) \coloneqq F(z) - \frac{\nabla f(y)^Tz}{\sqrt{\Vert\nabla f(y)\Vert^2 +1}} 
\end{align}
for some constant $a$ where $\displaystyle F(z) \coloneqq \int_a^z\frac{\nabla f(u)}{\sqrt{\Vert\nabla f(u)\Vert^2 +1}}du$ is a scalar potential function for the vector field integrand. Notice that $F$ is well defined since the integrand is a conservative vector field. Also, F(z) is convex since the integrand is a maximal monotone operator resulting from a convex, proper and lower semi-continuous function, $f$ \citep{minty1962monotone,brezis1973operateurs,rockafellar1970convex}.

Thus, $f_x$ and $f_y$ are well defined and convex since the difference of a convex function and a linear function is convex. $f_x$ is minimized at $z=x$, thus evaluating $f_x$ at $y$ and $x$ and subtracting the results gives
\begin{align}
F(y) &-F(x)-\bigg\langle \frac{\nabla f(x)}{\sqrt{\Vert\nabla f(x)\Vert^2 +1}},y-x \bigg\rangle \nonumber\\
&= f_x(y) - f_x(x) \nonumber\\
&\ge \frac{1}{2L_G} \Vert \nabla f_x(y)\Vert^2 \label{conseq-quadratic-bd-used}\\
&= \frac{1}{2L_G}\bigg\Vert\frac{\nabla f(y)}{\sqrt{\Vert\nabla f(y)\Vert^2 +1}} - \frac{\nabla f(x)}{\sqrt{\Vert\nabla f(x) \Vert^2 +1}} \bigg\Vert^2 \label{grad-fx}
\end{align}
Equation \ref{conseq-quadratic-bd-used} follows from Lemma \ref{conseq-quadratic-bd} and equation \ref{grad-fx} from taking the gradient of $f_x$ at $x$.

\noindent
Similarly, $z=y$ minimizes $f_y$  and 
\begin{align}
F(x) &-F(y)-\bigg\langle \frac{\nabla f(y)}{\sqrt{\Vert\nabla f(y)\Vert^2 +1}},x-y \bigg\rangle \nonumber\\
&= f_y(x) - f_y(y) \nonumber\\
&\ge \frac{1}{2L_G} \Vert \nabla f_y(x)\Vert^2 \nonumber\\
&= \frac{1}{2L_G}\bigg\Vert\frac{\nabla f(y)}{\sqrt{\Vert\nabla f(y)\Vert^2 +1}} - \frac{\nabla f(x)}{\sqrt{\Vert\nabla f(x) \Vert^2 +1}} \bigg\Vert^2 \label{grad-fy}
\end{align}
Adding the inequalities \ref{grad-fx} and \ref{grad-fy}, we obtain
\begin{align*}
\frac{1}{L_G}\left\Vert \frac{\nabla f(x)}{\sqrt{\Vert\nabla f(x)\Vert^2 +1}} - \frac{\nabla f(y)}{\sqrt{\Vert\nabla f(y) \Vert^2 +1}} \right\Vert^2 
& \le \bigg \langle \frac{\nabla f(x)}{\sqrt{\Vert\nabla f(x)\Vert^2 +1}}- \frac{\nabla f(y)}{\sqrt{\Vert\nabla f(y)\Vert^2 +1}},x-y \bigg\rangle.
\end{align*}
\end{proof}

\begin{theorem}
Let $f:\mathbb{R}^n \to \mathbb{R}$ be continuous and $\nabla f$ Lipschitz continuous with $\gamma \le \frac{1}{L} $ where $L$ is the Lipschitz constant for $\nabla f$. Assume $f$ attains an optimal value at $x^* = \arg\min_x f(x)$. Then $T$ defined in \eqref{operator} is a contraction map with contraction parameter $\alpha = \sqrt{1 + \gamma^2 L_G^2 - 2\gamma \frac{L_G^2}{L}}$ where $L_G$ is the Lipschitz constant for 
$\displaystyle \frac{\nabla f}{\sqrt{\Vert \nabla f \Vert^2 + 1}} $ and $\Vert \cdot \Vert$ is the Euclidean norm in $\mathbb{R}^n$.
\end{theorem}

\begin{proof}
It suffices to show that for some $\alpha \in [0,1)$,
\begin{equation}
    \Vert Tx - Ty \Vert \le \alpha \Vert x - y \Vert.
\end{equation}

\noindent
To this end, we compute as follows
\begin{align}
\Vert Tx - Ty \Vert^2  &= \bigg\Vert x - \frac{\gamma\nabla f(x)}{\sqrt{\Vert\nabla f(x)\Vert^2 +1}}-y + \frac{\gamma\nabla f(y)}{\sqrt{\Vert\nabla f(y)\Vert^2 +1}}\bigg\Vert^2 \nonumber\\
&=\bigg\Vert (x -y)- \gamma\left ( \frac{\nabla f(x)}{\sqrt{\Vert\nabla f(x) \Vert^2 +1}} - \frac{\nabla f(y)}{\sqrt{\Vert\nabla f(y)\Vert^2 +1}} \bigg )\right\Vert^2 \nonumber\\
&= \Vert x -y\Vert^2 - 2\gamma \left \langle x -y,\frac{\nabla f(x)}{\sqrt{\Vert\nabla f(x)\Vert^2 +1}} - \frac{\nabla f(y)}{\sqrt{\Vert\nabla f(y)\Vert^2 +1} }\right \rangle \nonumber\\
&+\gamma^2\bigg\Vert \frac{\nabla f(x)}{\sqrt{\Vert\nabla f(x)\Vert^2 +1}} - \frac{\nabla f(y)}{\sqrt{\Vert\nabla f(y) \Vert^2 +1}} \bigg\Vert^2 \nonumber\\
 & \leq \bigg\Vert x -y\bigg\Vert^2+\bigg ( \gamma^2- \frac{2\gamma}{L_G} \bigg )\Bigg \Vert \frac{\nabla f(x)}{\sqrt{\Vert\nabla f(x)\Vert^2 +1}} - \frac{\nabla f(y)}{\sqrt{\Vert\nabla f(y)\Vert^2 +1}} \Bigg\Vert^2 \label{co-coercivity-result}\\
&\leq \Vert x - y \Vert^2 + L_G^2 \bigg ( \gamma^2- \frac{2\gamma}{L_G} \bigg ) \Vert x -y \Vert^2 \label{Lipschitzcc}\\
& \le  \bigg (1 + \gamma^2 L_G^2- \frac{2 \, \gamma L_G^2}{L} \bigg ) \Vert x -y \Vert^2 
\end{align}
Inequality \eqref{co-coercivity-result} follows from co-coercivity (see Lemma \ref{cocoer}) while \ref{Lipschitzcc} follows from the continuity of $\displaystyle \frac{\nabla f}{\sqrt{\Vert \nabla f \Vert^2 + 1}} $ and the Lipschitz continuity of $\nabla f$.
\end{proof}

\section{Stochastic Convergence Proof}\label{proof-stochastic}
\subsection{Important Lemmas and Definitions}
In this section, we provide proof of convergence of our algorithm. To this end, we start with some definitions and lemmas necessary for the main theorem.
\begin{definition}
A differentiable function $f:\mathbb{R}^n \to \mathbb{R}$ is convex if for all $x,y \in \mathbb{R}$, 
\begin{align}
    f(y) \ge f(x) + \nabla f(x)^T (y-x).
\end{align}
\end{definition}

\noindent
For the rest of the paper, we make the following assumptions on the stochastic objective function $f_t:\mathbb{R}^n \to \mathbb{R}$ where the iteration counter $t\ge 1$.
\begin{assumptions}\label{stochastic-assume}
\begin{enumerate}
\item $f_t$ is convex for all $t$.
\item For all $t, f_t$ is differentiable.
\item For all $t$, there exists $G\ge 0$ such that $\Vert\nabla f_t(x)\Vert \le D$ for all $x \in X \subseteq \mathbb{R}^n$ where $X$ is the feasible set.
\item $A\coloneqq \{x_1,x_2, \cdots\}$, the iterates generated by GeoAdaLer algorithm.
\item $ x^* \coloneqq \arg\min_{x\in X} \sum_{t=1}^T f_t(x)$ exists and $\Vert x_k - x^*\Vert \le D$ for all $x_k \in A$.
\item $\beta_t \coloneqq \beta \lambda^{t-1}$ where $\lambda \in (0,1), \beta \in [0,1)$.
\end{enumerate}
\end{assumptions}
\noindent
Also, for notational convenience, we take $g_t = \nabla f_t(x_t)$.

\begin{lemma}\label{mt-bounded}
Under Assumptions \ref{stochastic-assume}, no. 3, the exponential moving average
\begin{align}
m_t = \beta_t m_{t-1} +  (1-\beta_t) g_t
\end{align}
is bounded for all  $t$.
\end{lemma}
\begin{proof}
Iteratively expanding out $m_t$, we obtain
\begin{align*}
m_t = (1-\beta_t) \sum_{i=1}^t \beta_t^{t-i+1} g_{i}.
\end{align*}
Taking the Euclidean norm on both sides and using the boundedness of $g_t$ gives
\begin{align*}
\Vert m_t \Vert &= \Vert (1-\beta_t) \sum_{i=1}^t \beta_t^{t-i+1} g_{i}  \Vert\\
                &\leq (1-\beta_t) \sum_{i=1}^t \beta_t^{t-i+1} \Vert g_{i} \Vert\\
                & \leq (1-\beta_t) G \sum_{i=1}^t \beta_t^{t-i+1}\\
                & = G (1 - \beta_t^t) \leq G\\
\end{align*}
\end{proof}

\begin{lemma} \label{beta-bounds}
Under Assumptions \ref{stochastic-assume}, no. 6, the following inequalities hold
\begin{enumerate}
\item $\displaystyle \frac{\beta_t}{1-\beta_t} \leq \frac{\beta}{1-\beta}$
\item $\displaystyle \frac{1}{1-\beta_t} \leq \frac{1}{1-\beta}$
\item $\displaystyle \sum_{t=1}^T \frac{\beta_t}{1-\beta_t} \le  \frac{\beta(1-\lambda^T)}{(1-\beta)(1-\lambda)}$
\end{enumerate}
for all $t$.
\end{lemma}
\begin{proof}
1). 
For all $t, \lambda < 1$ gives $\lambda^{t-1} < 1$ so that $\frac{1}{\lambda^{1-t}} < 1$. Thus
\begin{align*}
\frac{\beta_t}{1-\beta_t} &= \frac{\beta \lambda^{t-1}}{1-\beta \lambda^{t-1}} = \frac{\beta}{\lambda^{1-t} - \beta} \leq \frac{\beta}{1-\beta}.
\end{align*}

2). $\beta_t = \beta \lambda^{t-1} \le \beta$ for all $t$ since $\lambda \in (0,1), \beta \in [0,1)$. So, $1-\beta \le 1-\beta \lambda^{t-1}$ leads to 
$\displaystyle \frac{1}{1-\beta_t} \leq \frac{1}{1-\beta}$ as required.

3). By Lemma \ref{beta-bounds} no. 2, we have
\begin{align*}
\sum_{t=1}^T \frac{\beta_t}{1-\beta_t} \leq \sum_{t=1}^T \frac{\beta_t}{1-\beta} = \frac{\beta }{1-\beta} \sum_{t=1}^T \lambda^{t-1} 
= \frac{\beta(1-\lambda^T)}{(1-\beta)(1-\lambda)}
\end{align*}
\end{proof}

\begin{lemma}\label{intergral_test}
For all $t \ge 1$ and for all $T \ge t$, the following inequality holds
\begin{align}
\sum_{t=1}^T \frac{1}{\sqrt{t}} \le 1 + \int_1^T \frac{1}{\sqrt{t}}dt
\end{align}
\end{lemma}
Proof trivially follows from the integral test for convergence of series and is also given in \cite{bock2018improvement}.

\begin{lemma}
Under Assumptions \ref{stochastic-assume}, no. 1 and 2, the following inequality holds for all $t$
\begin{align}
R(T) \le \sum_{t=1}^T g_t \cdot  (x_t - x^*).
\end{align}
\end{lemma}

\begin{proof}
By converxity of $f_t$ for each $t$,
\begin{align*}
f_t(x^*) - f_t(x_t) \ge \nabla f_t (x_t) \cdot (x^* - x_t)
\end{align*}
It then follows that
\begin{align*}
 f_t(x_t) - f_t(x^*) \le \nabla f_t (x_t) \cdot (x_t - x^*) = g_t \cdot (x_t - x^*)
\end{align*}
Hence, summing both sides from $t=1$ to $T$ gives the required result.
\end{proof}

\section{Proof of Theorem \ref{geo-convergence}}\label{convergence-proof}
From the update rule in Algorithm \ref{geo-algorithm}
\begin{align*}
x_{t+1} = x_t - \gamma_t \frac{m_t}{\sqrt{\Vert m_t\Vert^2 +1}}.
\end{align*}
Subtracting $x^*$ from both sides and taking norm squared results
\begin{align*}
\Vert x_{t+1} - x^*\Vert^2 &= \vert x_t - \gamma_t \frac{m_t}{\sqrt{\Vert m_t\Vert^2 +1}} \vert^2\\
        &= \Vert x_{t} - x^*\Vert^2 -\frac{2\gamma_t}{\sqrt{\Vert m_t\Vert^2 +1}}
        m_t \cdot (x_t - x^*) + \gamma_t^2 \frac{\Vert m_t\Vert^2}{\Vert m_t\Vert^2 + 1}.
\end{align*}
Substituting $m_t = \beta_t m_{t-1} +  (1-\beta_t) g_t $, we obtain
\begin{align*}
\Vert x_{t+1} - x^*\Vert^2 &=
\begin{multlined}[t]
\Vert x_{t} - x^*\Vert^2 -\frac{2\gamma_t}{\sqrt{\Vert m_t\Vert^2 +1}} m_{t-1} \cdot (x_t - x^*)\\
- 
\frac{2\gamma_t (1-\beta_t)}{\sqrt{\Vert g_t\Vert^2 +1}} g_{t} \cdot (x_t - x^*)
+ \gamma_t^2 \frac{\Vert m_t\Vert^2}{\Vert m_t\Vert^2 + 1}.
\end{multlined}
\end{align*}
Rearrange to have $g_t\cdot (x_t - x^*)$ on the left hand side:
\begin{align}
g_t\cdot (x_t - x^*) &= \begin{multlined}[t]
\frac{\sqrt{\Vert m_t\Vert^2+1}}{2\gamma_t (1-\beta_t)}
\left[\Vert x_t-x^*\Vert^2 - \Vert x_{t+1} -x^*\Vert^2\right]
-
\frac{\beta_t}{1-\beta_t} m_{t-1} \cdot (x_t-x^*) \\
+ \frac{\gamma_t}{2(1-\beta_t)} \frac{\Vert m_t\Vert^2}{\sqrt{\Vert m_t \Vert^2+1}}
\end{multlined} \nonumber\\
& \le 
\begin{multlined}[t]
\frac{\sqrt{\Vert m_t\Vert^2+1}}{2\gamma_t (1-\beta_t)}
\left[\Vert x_t-x^*\Vert^2 - \Vert x_{t+1} -x^*\Vert^2\right]\\
+
\frac{\beta_t}{1-\beta_t} \Vert m_{t-1}\Vert \Vert (x^*-x_t)\Vert
+ \frac{\gamma_t \Vert m_t\Vert^2}{2(1-\beta_t) }. \label{Cauchy-S}
\end{multlined}
\end{align}

where inequality \ref{Cauchy-S} follows from the Cauchy-Schwartz inequality applied to the second term $-m_{t-1} \cdot (x_t-x^*) = m_{t-1} \cdot (x^*-x_t)$ and the fact that $\frac{\Vert m_t\Vert^2}{\sqrt{\Vert m_t \Vert^2+1}} \le 1$ for all $t$.\\
Summing both sides from $t=1$ to $T$ and using Lemma \ref{mt-bounded} and Assumptions \ref{stochastic-assume} no. 3 and 5, we further obtain

\begin{align}
\sum_{t=1}^T g_t\cdot (x_t - x^*)  
& \le 
\begin{multlined}[t]
\sum_{t=1}^T \frac{\sqrt{G^2+1}}{2\gamma_t (1-\beta_t)}
\left[\Vert x_t-x^*\Vert^2 - \Vert x_{t+1} -x^*\Vert^2\right] \nonumber\\
+
DG\sum_{t=1}^T \frac{\beta_t}{1-\beta_t} 
+ G^2\sum_{t=1}^T \frac{\gamma_t }{2(1-\beta_t) } 
\end{multlined}\\
&\le
\begin{multlined}[t]
\sum_{t=1}^T \frac{\sqrt{G^2+1}}{2\gamma_t (1-\beta_t)}
\left[\Vert x_t-x^*\Vert^2 - \Vert x_{t+1} -x^*\Vert^2\right] \label{second-last} \\
+
\frac{DG\beta(1-\lambda^T)}{(1-\beta)(1-\lambda)}
+  \frac{ G^2}{2(1-\beta) } \sum_{t=1}^T \gamma_t.
\end{multlined}
\end{align}

where the last inequality follows from Assumptions \ref{stochastic-assume} no. 3 and Lemma \ref{beta-bounds} no. 2 and 3.\\
By expanding the first term and and rewriting the what is left in compact form, we obtain

\begin{align*}
\sum_{t=1}^T \frac{1}{\gamma_t (1-\beta_t)}
\left[\Vert x_t-x^*\Vert^2 - \Vert x_{t+1} -x^*\Vert^2\right]
&= 
\begin{multlined}[t]
\frac{1}{\gamma_1 (1-\beta_1)} \Vert x_1 - x^*\Vert^2 -
\frac{1}{\gamma_T (1-\beta_T)} \Vert x_{T+1} - x^*\Vert^2 \nonumber\\
+ \sum_{t=2}^T \left(\frac{1}{\gamma_t (1-\beta_t) } - \frac{1}{\gamma_{t-1} (1-\beta_{t-1})} \right)\Vert x_t - x^*\Vert^2
\end{multlined}\\
&\le 
\begin{multlined}[t]
\frac{\Vert x_1 - x^*\Vert^2}{\gamma_1 (1-\beta_1)}  \\
+
\sum_{t=1}^T \left(\frac{1}{\gamma_t (1-\beta_t) } - \frac{1}{\gamma_{t-1} (1-\beta_{t-1})} \right) \Vert x_t - x^*\Vert^2.
\end{multlined}\\
\end{align*}

where the last inequality follows from dropping the second term. By the Assumptions \ref{stochastic-assume} no. 5, it further simplifies to
\begin{align*}
\sum_{t=1}^T \frac{1}{\gamma_t (1-\beta_t)}
\left[\Vert x_t-x^*\Vert^2 - \Vert x_{t+1} -x^*\Vert^2\right] 
&\le
\frac{D^2}{\gamma_1(1-\beta_1)}  + 
D^2\sum_{t=2}^T \left(\frac{1}{\gamma_t (1-\beta_t) } - \frac{1}{\gamma_{t-1} (1-\beta_{t-1})} \right)\\
&=
\frac{D^2}{\gamma_1(1-\beta_1)}  - \frac{D^2}{\gamma_1(1-\beta_1)} + \frac{D^2}{\gamma_T(1-\beta_T)}\\
&=\frac{D^2}{\gamma_T(1-\beta_T)}.
\end{align*}
From Lemma \ref{beta-bounds} we have.
$$\sum_{t=1}^T \frac{1}{\gamma_t (1-\beta_t)}
\left[\Vert x_t-x^*\Vert^2 - \Vert x_{t+1} -x^*\Vert^2\right] 
\le \frac{D^2}{\gamma_T(1-\beta)}.$$
Therefore equation \ref{second-last} become:

\begin{align}
\sum_{t=1}^T g_t\cdot (x_t - x^*)
    &\le
    \frac{D^2\sqrt{G^2+1}}{\gamma_T(1-\beta_T)} 
    +
    \frac{DG\beta(1-\lambda^T)}{(1-\beta)(1-\lambda)}
    +  \frac{ G^2}{2(1-\beta) } \sum_{t=1}^T \gamma_t.
\end{align}

Assuming $\gamma_t=\frac{1}{\sqrt{t}}$ we further obtain:

\begin{align}
\sum_{t=1}^T g_t\cdot (x_t - x^*)
    &\le
    \frac{\sqrt{T}D^2\sqrt{G^2+1}}{(1-\beta_T)} 
    +
    \frac{DG\beta(1-\lambda^T)}{(1-\beta)(1-\lambda)}
    +  \frac{ G^2}{2(1-\beta) } \sum_{t=1}^T \frac{1}{\sqrt{t}}\\
    &\le
    \frac{\sqrt{T}D^2\sqrt{G^2+1}}{(1-\beta_T)} 
    +
    \frac{DG\beta(1-\lambda^T)}{(1-\beta)(1-\lambda)}
    +  \frac{ G^2 (2T-1)}{2(1-\beta) }. \label{last}
\end{align}
Equation \ref{last} follows from Lemma \ref{intergral_test} and so our regret is bounded above by:
$$R(T)\le    \frac{\sqrt{T}D^2\sqrt{G^2+1}}{(1-\beta_T)} 
    +
    \frac{DG\beta(1-\lambda^T)}{(1-\beta)(1-\lambda)}
    +  \frac{ G^2 (2T-1)}{2(1-\beta) }.$$

\section{Datasets \label{datasets}}

\textbf{MNIST:} The MNIST database of handwritten digits. Licensed under the Creative Commons Attribution 4.0 License\citep{lecun2010mnist}.

\textbf{CIFAR-10:} The CIFAR-10 dataset consists of 60000 32x32 colour images in 10 classes, with 6000 images per class. There are 50000 training images and 10000 test images. Licensed under the Creative Commons Attribution 4.0 License \citep{Krizhevsky09learningmultiple}.

\textbf{Fashion MNIST:} The Fashion MNIST database of  fashion images. Licensed under The MIT License (MIT).\cite{xiao2017fashion}

\end{document}